\pdfoutput=1

\documentclass[11pt]{article}

\usepackage[final]{acl}

\usepackage{times}
\usepackage{latexsym}

\usepackage[T1]{fontenc}

\usepackage[utf8]{inputenc}

\usepackage{microtype}

\usepackage{inconsolata}

\usepackage{microtype}
\usepackage{graphicx}
\usepackage{booktabs} 
\usepackage{hyperref}
\usepackage{rycolab}
\usepackage{macros}
\usepackage{caption}
\usepackage{subcaption}
\usepackage{multicol,lipsum}

\title{Direct Preference Optimization with an Offset}
\author{Afra Amini \qquad Tim Vieira \qquad Ryan Cotterell \\
\setlength{\fboxsep}{2.5pt}%
\setlength{\fboxrule}{2.5pt}%
\fcolorbox{white}{white}{
    $\{$\texttt{\href{mailto:afra.amini@inf.ethz.ch}{afra.amini}, }\texttt{\href{mailto:ryan.cotterell@inf.ethz.ch}{ryan.cotterell}}$\}$\texttt{@inf.ethz.ch} \quad
\texttt{\href{mailto:tim.f.vieira@gmail.com}{tim.f.vieira@gmail.com}}
} \\
    {%
\setlength{\fboxsep}{2.5pt}%
\setlength{\fboxrule}{2.5pt}%
\fcolorbox{white}{white}{
    \includegraphics[width=.15\linewidth]{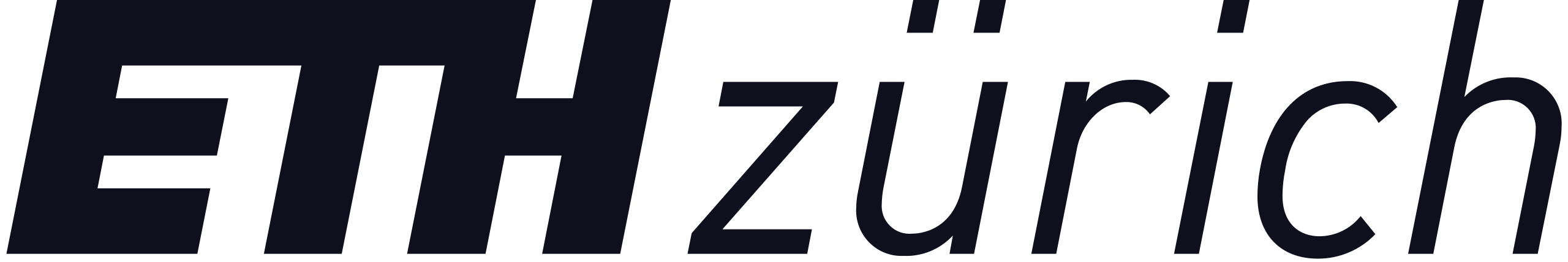}
}
}}
\begin{document}
\maketitle

\begin{abstract}
Direct preference optimization (DPO) is a successful fine-tuning strategy for aligning large language models with human preferences \emph{without} the need to train a reward model or employ reinforcement learning. 
DPO, as originally formulated, relies on binary preference data and fine-tunes a language model to increase the likelihood of a preferred response over a dispreferred response.
However, not all preference pairs are equal. Sometimes, the preferred response is only slightly better than the dispreferred one. In other cases, the preference is much stronger. For instance, if a response contains harmful or toxic content, the annotator will have a strong preference for that response.
In this paper, we propose a generalization of DPO, termed \defn{DPO with an offset (ODPO)}, that does not treat every preference pair equally during fine-tuning. Intuitively, ODPO requires the difference between the likelihood of the preferred and dispreferred response to be greater than an offset value. The offset is determined based on the extent to which one response is preferred over another. 
Our experiments on various tasks suggest that ODPO significantly outperforms DPO in aligning language models, especially when the number of preference pairs is limited. 
\newline
 \newline
 \vspace{0.2em}
  \hspace{.5em}\includegraphics[width=1.25em,height=1.25em]{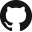}\hspace{.75em}\parbox{\dimexpr\linewidth-2\fboxsep-2\fboxrule}
  {\url{https://github.com/rycolab/odpo}}
\end{abstract}

\section{Introduction}
Reinforcement learning from human feedback (RLHF) is a key building block in training the most modern large language models.
The algorithm aligns the language model's responses to human preferences \citep{instructgpt, llama2, palm2, gpt4}. 
A typical implementation of RLHF is as follows. 
First, humans compare paired responses from a language model, i.e., they determine which of the responses is better, e.g., more helpful or less toxic. 
Second, a reward model is trained to give higher rewards to the responses preferred by humans. 
Lastly, a reinforcement learning algorithm is used to update the model's parameters to maximize the expected reward, while not diverging too much from the model's initial parameters.

The aforementioned implementation of RLHF, however, is tedious for two reasons. 
First, the reward model is usually a model as large as the language model itself, which is expensive to train and store. 
Second, reinforcement learning algorithms are known to be very sensitive to the choice of hyperparameters \citep{zheng2023secrets}, and, thus, hard to tune. 
Therefore, a fruitful line of work attempts to design alternative simpler implementations of RLHF \citep{unlikelihood, quark, slic-hf}.\looseness=-1

\begin{figure*}
    \centering
    \includegraphics[width=\textwidth]{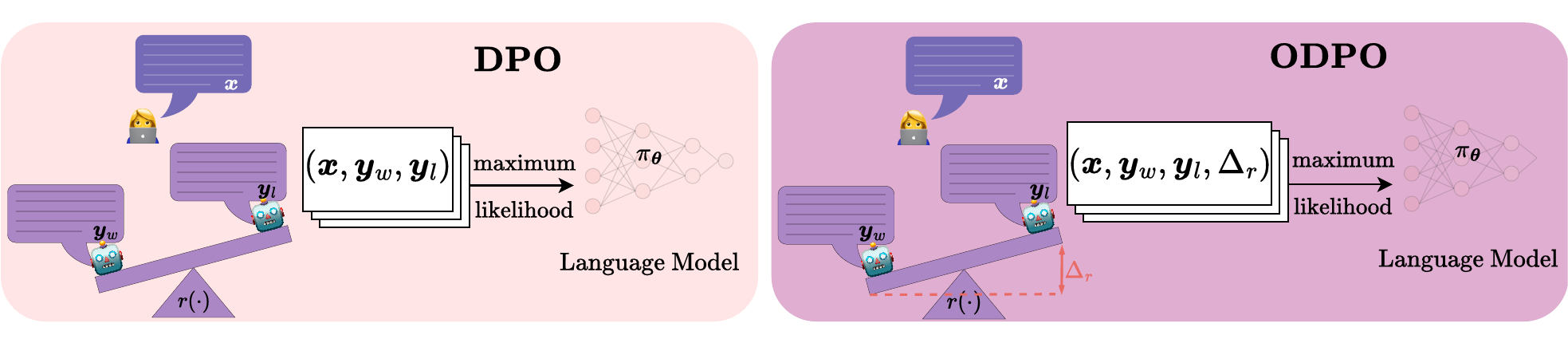}
    \caption{ODPO takes into account the extent to which one output should be preferred over another. The model has to put more probability mass on the preferred output compared to the dispreferred output by an offset that is determined based on how much the winning output is preferred over the losing output.}
    \label{fig:fig1}
\end{figure*}

A highly successful and straightforward alternative to RLHF is \defn{direct preference optimization} \citep[DPO;][]{dpo}. 
Unlike the conventional implementation of RLHF, DPO does not train a reward model and sidesteps the use of any reinforcement learning algorithm. Instead, it fine-tunes the language model's weights to maximize the likelihood of the preference data directly using the Bradley--Terry model \citep{bt}. 
Intuitively, each DPO gradient update to the model's parameters increases the likelihood of the preferred response and decreases the likelihood of the dispreferred response.

However, DPO only takes the ordering between the model's responses into account and not the extent to which one response is preferred over another. 
In many settings, the difference in quality between the two responses is known---either by asking humans \citep{summarizehf, llama2}\footnote{For example, \citet{llama2} ask humans to specify the extent to which they prefer one output over another, i.e., by choosing from: significantly better, better, slightly better, negligibly better / unsure.} or through point-wise scores, e.g., toxicity scores, sentiment scores, code-based execution scores, etc. 
In that context, a natural question is the following: how can we make use of this information to better align language models with human preferences? 

To answer this question, we propose a new method, which we term DPO with an offset (ODPO). ODPO is a generalization of DPO that incorporates the difference between responses when modeling preference data. 
The intuition behind ODPO is simple; it requires the language model to increase the likelihood of the preferred responses compared to the dispreferred responses by an offset that is determined based on the difference between their associated reward values. Therefore, the larger the reward of the preferred response in comparison to the dispreferred response, the higher the likelihood needs to be over the likelihood of the dispreferred response. 
We further show that when the offset is set to zero, ODPO is equivalent to DPO. 
This process is illustrated in \Cref{fig:fig1}.\looseness=-1

In our experiments, we fine-tune language models of various sizes with DPO and ODPO and compare the two methods. 
In cases where a ground-truth reward function is given, we measure two competing metrics: (i) the average reward associated with the generations from the fine-tuned model, and (ii) the KL divergence between the language model before and after the fine-tuning.
Based on the results of sentiment and toxicity control tasks, we observe that ODPO more often appears on the Pareto frontier of reward and KL compared to DPO. 
We then apply ODPO to a summarization task, where we use the scores given by humans on a Likert scale to define an offset between the two summaries. 
We observe that on average and across different sampling temperatures, ODPO results in a higher win rate over human-written summaries compared to DPO.\looseness=-1

\section{Preliminaries}
Given a prompt $\bx \in \alphabet^*$, a language model $\policy$ is a distribution over the responses $\by \in \alphabet^*$, where $\alphabet$ is the alphabet.\footnote{An alphabet is a finite, non-empty set.}
As is standard, we parameterize the conditional distribution over responses given a prompt as an autoregressive language model $\policy(\by \mid \bx)$ parameterized by $\btheta$ as,
\begin{equation}
\begin{aligned}
&\policy(\by \mid \bx) = \\
&\qquad \policy(\eos \mid \by, \bx) \prod_{t=1}^T \policy(y_t \mid \by_{<t}, \bx),
\end{aligned}
\end{equation}
where $\eos \notin \alphabet$ is a distinguished end-of-sequence token. 
To deploy a language model to a downstream task, it is usually necessary to additionally fine-tune it on high-quality data. 
Given a dataset of prompts and desired responses $(\bx, \by) \sim \sftdata$, the standard fine-tuning objective (to be maximized) is simply the log-likelihood of $\sftdata$ 
\begin{align} \label{eq:mle}
    \mlobj(\btheta) = \E_{(\bx, \by) \sim \sftdata}\!\Bigl[ \log \policy(\by \mid \bx) \Bigr].
\end{align}
We use $\piref$ to refer to the language model fine-tuned with this objective on $\sftdata$. 

Such a fine-tuning approach can be effective for solving downstream tasks with language models. 
However, there is a disparity between the maximum-likelihood objective and the goal of the downstream task, which is to leverage the language model to generate high-quality responses, as judged by humans.
This misalignment arises partly because the maximum-likelihood objective, \cref{eq:mle}, fails to distinguish between major and minor errors. 
For instance, when fine-tuning a language model to perform summarization, a significant error might involve hallucination, while a minor one might involve a failure to employ the appropriate synonyms \citep{summarizehf}. 
Training with the maximum-likelihood objective encourages the model to assign nonzero probability mass to all responses in $\sftdata$, even those of lower quality. 

To address this issue, one can optimize the expected reward of language model generations instead of using the maximum-likelihood objective. 
Consider a reward function $\reward(\bx, \by)$, estimated from quality assessments performed by humans.
 The reward function assigns real values to prompt--response pairs; we will discuss in \Cref{sec:reward} how such a reward model can be learned from human feedback. 
 Our objective is to make sure that the responses generated by the language model have high quality. Therefore, we can directly optimize the expected reward of the responses generated by the language model,
\begin{align}
    \rlobj(\btheta) = \E_{\bx \sim \dataset, \by \sim \policy(\cdot|\bx)}\Big[ r(\bx, \by) \Big],
\end{align}
where $\dataset = \{\bx^{(n)}\}_{n=1}^N$ is a multiset of prompts. To prevent reward hacking \citep{amodei2016concrete} and to make sure that we do not diverge too much from the supervised fine-tuned model $\piref$, a regularization term is often added to the objective \citep{summarizehf},
\begin{equation}
\begin{aligned}\label{eq:rlobj}
    \rlobj(\btheta) &= \E_{\bx \sim \dataset, \by \sim   \policy(\cdot|\bx)}\Big[ r(\bx, \by) \Big] \\
    &- \beta\, \kl\Big[\policy(\by \mid \bx) \, \| \, \piref(\by \mid \bx)\Big].
\end{aligned}
\end{equation}
The above objective is optimized using proximal policy optimization \citep[PPO;][]{ppo} or another actor--critic algorithm \citep{pmlr-v48-mniha16, glaese2022improving}.

\section{Reward Modeling}\label{sec:reward}

\paragraph{Pointwise Rewards. }
A key component in \Cref{eq:rlobj} is the task-dependent reward function that assigns pointwise real-valued rewards to each output. 
In many tasks, learning such reward functions is straightforward. 
For example, in open-ended text generation, a desired attribute could be the presence of indicators of positive sentiment in the generated text, while an undesired attribute could be toxicity. 
In such cases, the reward model might take the form of a classifier that assesses responses based on their sentiment or toxicity. 
Similarly, in code generation tasks, the quality of the code can be automatically evaluated, providing another straightforward example.\looseness=-1 

\paragraph{Pairwise Preferences.}
Learning a point-wise reward function for tasks like summarization or dialogue generation is more complex. 
Judging the absolute quality of a summary can depend on several factors, e.g., coherence, faithfulness, and conciseness, which makes it hard to collect human feedback datasets for reward model training.
An effective strategy in such cases is to collect human preferences instead of point-wise judgments \citep{ziegler, wu2021recursively, instructgpt}. 
In this setup, humans are shown two (or more) responses to a prompt $\bx$ and are asked to select the response they prefer. 
Therefore, a datapoint in a human feedback dataset is a triple $(\bx, \by_w, \by_l) \sim \hfdata$, where $\by_w$ is preferred over $\by_l$. 
Given such a dataset, one needs to learn the point-wise reward function. 
It is common to assume that the preference data can be modeled by a Bradley--Terry model \citep{bt},
\begin{subequations}
\begin{align} \label{eq:bt}
    \pbt(&\by_w \succ \by_l \mid \bx)  \nonumber  \\
    &= \frac{\exp\big(\reward(\bx, \by_w)\big)}{\exp\big(\reward(\bx, \by_w)\big) + \exp\big(\reward(\bx, \by_l)\big)} \\
    &= \sigma \big(\reward(\bx, \by_w) - \reward(\bx, \by_l) \big),
\end{align}
\end{subequations}
where $\sigma(x) = \frac{1}{1 + \exp(-x)}$ is the sigmoid function. 
Assuming that $\hfdata \sim \pbt$, we can train a reward model $\rphi$, parameterized by $\bphi$, as a binary classifier that maximizes the following log-likelihood\looseness=-1
\begin{align} \label{eq:loss-reward} 
      \E \Big[\log \sigma \big(\rphi(\bx, \by_w) - \rphi(\bx, \by_l) \big) \Big],
\end{align}
where the expectation is over $(\bx, \by_w, \by_l) \sim \hfdata$.
Intuitively, the reward function should assign higher reward values to the responses that are preferred by humans. 
The estimated reward function is then plugged into \Cref{eq:rlobj} to enable policy estimation.

\begin{figure*}
\begin{subequations}
\begin{align} 
   \dpoloss(\btheta) &= - \E_{(\bx, \by_w, \by_l) \sim \hfdata} \Big[\log \sigma \Big(\beta \log \frac{\policy(\by_w \mid \bx)}{\piref(\by_w \mid \bx)}  - \beta \log \frac{\policy(\by_l \mid \bx)}{\piref(\by_l \mid \bx)} \Big) \Big]  \\
   &= - \E_{(\bx, \by_w, \by_l) \sim \hfdata} \Big[\log \sigma \Big(\rhattheta(\bx, \by_w) - \rhattheta(\bx, \by_l)\Big)\Big] \label{eq:dpo-loss} \\
    \odpoloss(\btheta) &= - \E_{(\bx, \by_w, \by_l) \sim \hfdata} \Big[\log \sigma \Big(\rhattheta(\bx, \by_w) - \rhattheta(\bx, \by_l) - \deltar \Big)\Big] \label{eq:mdpo-loss}
\end{align}
\end{subequations}
\end{figure*}

It is often the case that directly optimizing \Cref{eq:rlobj} is complex and resource-intensive. 
Next, we will elaborate on an alternative method to using RLHF that has demonstrated comparable or even superior effectiveness to PPO.
\section{Direct Preference Optimization}
\citet{dpo} introduce a method to avoid reward model training and, thus, to directly optimize the language model. 
Their method, termed direct preference optimization (DPO) works as follows.
The critical observation is that the optimal solution that maximizes the \Cref{eq:rlobj} is,\looseness=-1
\begin{align}
    &\polstar(\by \mid \bx) = \\
    & \qquad \quad \frac{1}{Z(\bx)}\, \piref(\by \mid \bx) \exp \Big( \frac{1}{\beta} r(\bx, \by) \Big), \nonumber
\end{align}
where 
\begin{equation}
Z(\bx) = \sum_{\by \in \alphabet^*} \piref(\by\,\mid\,\bx) \exp \Big(\frac{1}{\beta} r(\bx, \by) \Big),
\end{equation}
Note that $Z(\bx)$ is, in general, difficult to compute as it involves a sum over a countably infinite set.
Nevertheless, we can write the reward as a function of $\polstar$ by rearranging the terms,
\begin{align} \label{eq:dpo-reward}
    \reward(\bx, \by) = \beta \log \frac{\polstar(\by \mid \bx)}{\piref(\by \mid \bx)} + \beta \log Z(\bx).
\end{align}
Under the assumption that the preference data is well-modeled by the Bradley--Terry model (\Cref{eq:bt}), we substitute the reward in \Cref{eq:loss-reward} with \Cref{eq:dpo-reward}, and formulate the loss function in terms of the language model parameters. Therefore, we directly optimize the language model by maximizing the likelihood of the preference data in \Cref{eq:dpo-loss}. In this equation, $\rhattheta(\bx, \by) = \beta \log \frac{\policy(\by|\bx)}{\piref(\by|\bx)}$ is called the \defn{estimated reward}. Note that the intractable normalization factor $Z(\bx)$ cancels out. 

Simply put, \Cref{eq:dpo-loss} requires the estimated reward for the preferred response to be larger than the estimated reward for the dispreferred response. 
A strong assumption behind the DPO's loss is that the preference data is well-modeled by a Bradley--Terry model. 
However, the Bradley--Terry model only tells us the probability that one response is preferred over another, and not the \emph{extent} to which this preference will hold. 
To enhance DPO, we will introduce our modification to DPO, which generalizes DPO and introduces an offset between the responses in its loss function.

\section{DPO with an Offset}
The intuition behind ODPO is simple: we want the estimated reward for the preferred response to be larger than the estimated reward for the dispreferred response, by an offset that is a function of the actual reward differences assigned to the two responses. 
To begin our exposition of ODPO, we first discuss the connection between the Bradley--Terry model and Gumbel random variables, which we then use to construct the loss function for ODPO.
\begin{figure*}[t]
     \centering
     \begin{subfigure}[b]{0.3\textwidth}
         \includegraphics[width=\linewidth]{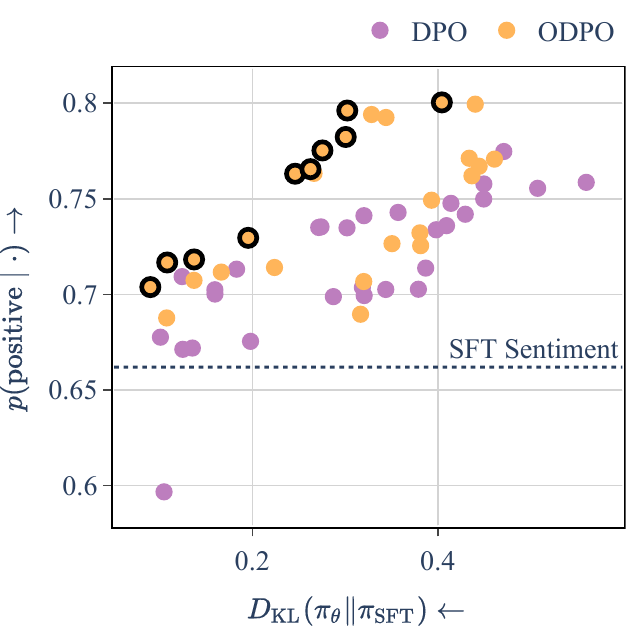}
         \caption{$|\hfdata| = 5000$, $100\%$ of the points on the Pareto front belong to ODPO.} %
     \end{subfigure}
     ~
     \begin{subfigure}[b]{0.3\textwidth}
         \includegraphics[width=\linewidth]{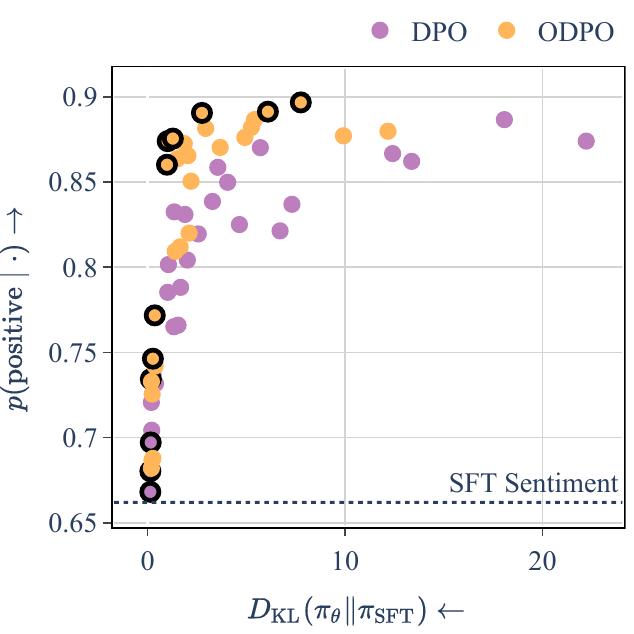}
         \caption{$|\hfdata| = 7500$, $66\%$ of the points on the Pareto front belong to ODPO.} %
     \end{subfigure}
     ~
     \begin{subfigure}[b]{0.3\textwidth}
         \includegraphics[width=\linewidth]{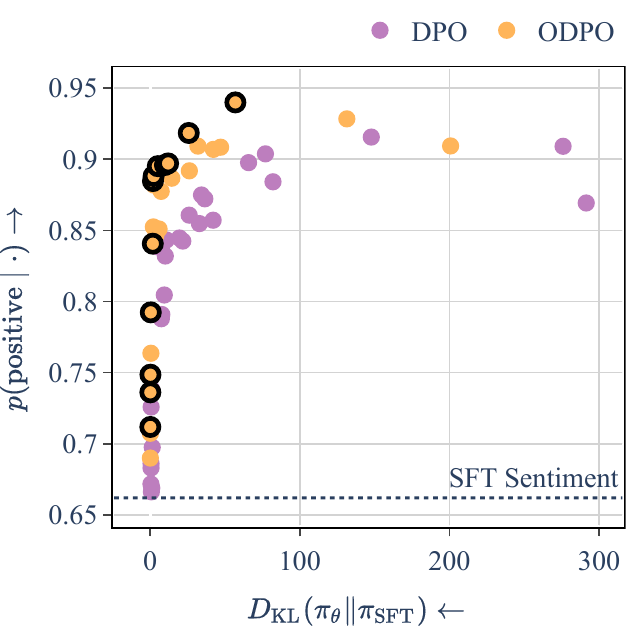}
         \caption{$|\hfdata| \mathbin{=} 10000$, $100\%$ of the points on the Pareto front belong to ODPO.} %
    \end{subfigure}
    \caption{Steering generated movie reviews towards positive sentiment. Points on the Pareto front are highlighted with a black border. We observe that in all $3$ settings, most (if not all) points on the Pareto front belong to ODPO.}
    \label{fig:imdb}
\end{figure*}
\begin{restatable}{theorem}{gumbelbt}\label{theorem:gumbelbt}
Let $\by_w$ and $\by_l$ be two responses to a prompt $\bx$, and let $\rhattheta(\bx, \by_w)$ and $\rhattheta(\bx, \by_l)$ be their associated estimated rewards.
Finally, let $\rtilde_w \sim \gumbel(\rhattheta(\bx, \by_w), 1)$ and $\rtilde_l  \sim \gumbel(\rhattheta(\bx, \by_l), 1)$ be Gumbel random variables.
Then, we have,
    \begin{align}
        \prob(\rtilde_w - \rtilde_l > 0) &= \pbt(\by_w \succ \by_l \mid \bx) \nonumber \\
        &= \sigma(\deltarhat),
    \end{align}
    where $\pbt(\by_w \succ \by_l \mid \bx)$ is a Bradley--Terry model \Cref{eq:bt} parameterized by $\rhattheta(\bx, \by_w)$ and $\rhattheta(\bx, \by_l)$, and $\deltarhat \defeq \rhattheta(\bx, \by_w) - \rhattheta(\bx, \by_l)$ is the difference between the estimated rewards.\looseness=-1
\end{restatable}
\begin{proof}
The proof is simple and follows directly from the Gumbel-max trick \citep{gumbel-trick}. 
See \Cref{sec:gumbel} for details.
\end{proof}
\begin{theorem}
    Let $\by_w$ and $\by_l$ be two responses for a prompt $\bx$.
    Given a desired offset $\deltar$ between the responses' reward values, 
    \begin{align}
        \prob\big(\rtilde_w - \rtilde_l > \deltar \big) = \sigma(\deltarhat - \deltar).
    \end{align}
    where $\deltarhat \defeq \rhattheta(\bx, \by_w) - \rhattheta(\bx, \by_l)$.
\end{theorem}
\begin{proof}
The difference between two independent Gumbel variables $\rtilde_w \sim \gumbel(\rhattheta(\bx, \by_w), 1)$ and $\rtilde_l \sim \gumbel(\rhattheta(\bx, \by_l), 1)$ follows the logistic distribution. Therefore, the probability of the difference between the noisy estimated rewards being greater than $\deltar$ is,
\begin{align} \label{eq:logistic}
    \prob\big(\rtilde_w - \rtilde_l > \deltar \big) = 1 - \calF\big( \deltar \big),
\end{align}
where $\calF(\cdot)$ is the CDF for logistic distribution. The mean of this distribution is $\deltarhat$, and the variance is $1$. Substituting the CDF in \Cref{eq:logistic} with its definition,
\begin{subequations}
\begin{align} 
    \prob\big(\rtilde_w &- \rtilde_l > \deltar\big) = 1 - \calF(\deltar)  \\
    &= \frac{1}{2} - \frac{1}{2} \mathrm{tanh}\Big(\frac{1}{2} (\deltar - \deltarhat) \Big) \\
    &= \frac{1}{2} - \frac{1}{2} \Big(2\sigma(\deltar - \deltarhat) - 1 \Big)\\
    &= 1 - \sigma(\deltar - \deltarhat)  \\
    &= \sigma(\deltarhat - \deltar).
\end{align}
\end{subequations}
\end{proof}
Similar to DPO, we then maximize the likelihood of the preference data, as shown in \Cref{eq:mdpo-loss}.
Intuitively, minimizing the above loss function enforces the estimated reward for the preferred response $\by_w$ to be higher than $\by_l$ by the offset $\deltar$. 
Comparing \Cref{eq:dpo-loss} and \Cref{eq:mdpo-loss}, when $\deltar = 0$, ODPO will be equivalent to DPO.

\paragraph{Designing $\deltar$.} 
We choose to model the offset as a monotonically increasing function $\sFunc(\cdot)$ of the difference between the scores associated with the responses:
\begin{align} \label{eq:margin}
    \deltar = \alpha\, \sFunc\big(\score(\bx, \by_w) - \score(\bx, \by_l) \big).
\end{align}
where $\alpha$ is a hyperparameter that controls the extent to which an offset should be enforced. When $\alpha = 0$, ODPO becomes equivalent to DPO. In tasks where a ground-truth reward model $\reward(\cdot)$ is given we set $\score(\cdot) = \reward(\cdot)$. In other tasks, one can use the quality scores given by humans to responses as $\score$ (see \Cref{sec:summarization}), or ask humans directly for the difference between the quality of the responses \citep{llama2}.  
We ablate $\sFunc$ in \Cref{sec:ablate-f} and $\alpha$ in \Cref{sec:ablate-alpha}. \looseness=-1

\paragraph{Connection to Softmax Margin.} 
We now show how ODPO is connected to softmax margin \citep{gimpel-smith-2010-softmax}.
The idea behind the softmax margin is to augment the softmax with a $\cost$ function, such that high-cost responses get penalized more heavily. 
For two responses per prompt, we maximize
\begin{align} \label{eq:bt-cost}
    \E_{(\bx, \by_w, \by_l)} \Bigl[\log \sigma\Big( \deltarhat - \cost(\by_w, \by_l)\Big)\Bigr].
\end{align}
Importantly, the $\cost$ function is by definition a non-negative function. 
Therefore, ODPO loss \Cref{eq:mdpo-loss} is only equivalent to softmax-margin loss if we restrict $\deltar \geq 0$.
\begin{figure*}
     \centering
     \begin{subfigure}[b]{0.3\textwidth}
         \includegraphics[width=\linewidth]{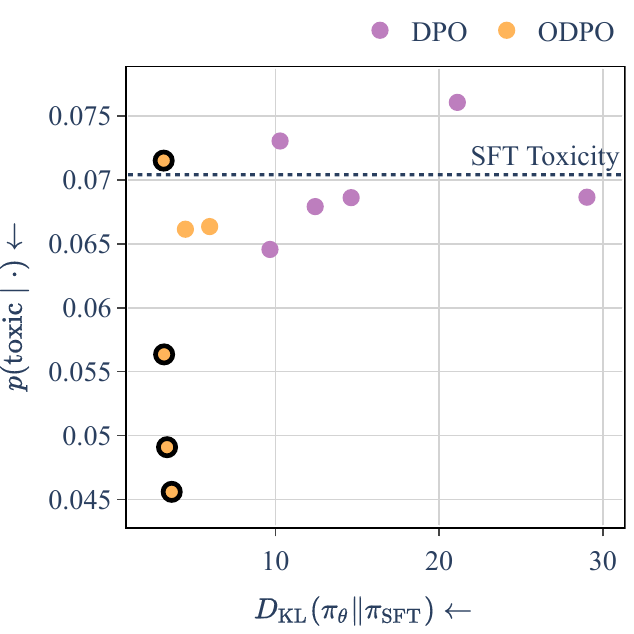}
         \caption{$|\hfdata| = 8000$, $100\%$ of the points on the Pareto front belong to ODPO.} %
     \end{subfigure}
     ~
     \begin{subfigure}[b]{0.3\textwidth}
         \includegraphics[width=\linewidth]{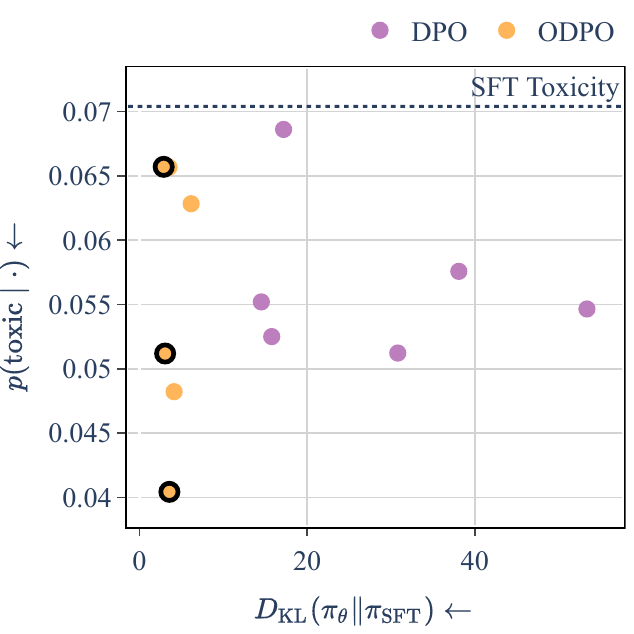}
         \caption{$|\hfdata| = 9000$, $100\%$ of the points on the Pareto front belong to ODPO.} %
     \end{subfigure}
     ~
     \begin{subfigure}[b]{0.3\textwidth}
         \includegraphics[width=\linewidth]{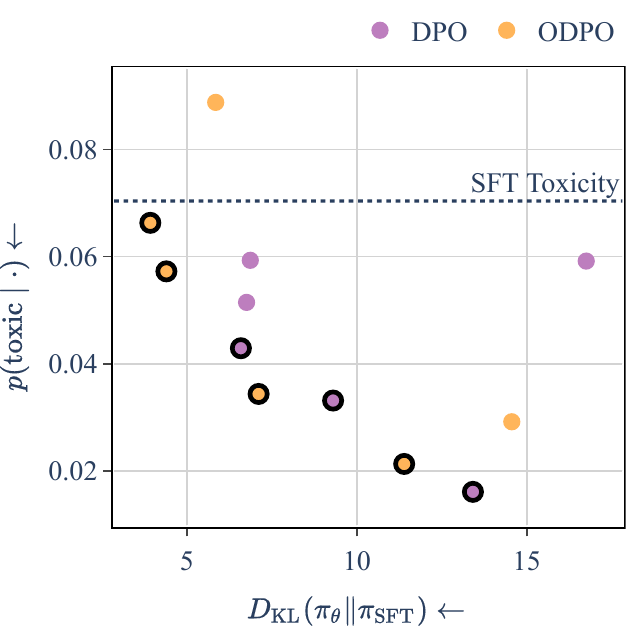}
         \caption{$|\hfdata| = 10000$, $57\%$ of the points on the Pareto front belong to ODPO.} %
    \end{subfigure}
    \caption{Steering generations away from toxic content. We highlight points on the Pareto front with a black border. We observe that, especially when the size of the dataset is small, ODPO manages to reduce the toxicity better than DPO while not diverging too far from the SFT model.
   }
    \label{fig:toxicity}
\end{figure*}
\section{Experiments}
In this section, we empirically compare the performance of ODPO with DPO across different tasks. 
We refer the reader to \Cref{sec:comp-budget} for the computational budget used for conducting the experiments.
\subsection{Sentiment Control}
In this experiment, we steer generations of \texttt{GPT2-Large} \citep{radford2019language} model towards positive sentiment. Following the experimental setup in \citep{dpo}, we fine-tune \texttt{GPT2-Large} on the train split of \textsc{imdb} dataset \citep{imdb} until convergence. This gives us the SFT checkpoint $\piref$. 

\paragraph{Reward Model.} 
We use a state-of-the-art binary sentiment classifier\footnote{Specifically, we use \url{https://huggingface.co/lvwerra/distilbert-imdb}.} with sentiments $\{\pos, \negg\}$ as the reward model.
Concretely if the sentiment of the response is negative, we set $r(\bx, \by)$ to $1-p(\negg \mid \cdot)$, and if the sentiment is positive $r(\bx, \by) = 1 + p(\pos \mid \cdot)$, where $p$ is given by the classifier. For notational ease, we show $\reward(\bx, \by)$ with $\reward(\by)$.\looseness=-1

\paragraph{Bootstraping the Preference Data.} 
DPO, as opposed to RLHF, only works on top of preference data. Therefore, in tasks that a ground-truth reward model is given, the preference data needs to be bootstraped.  We follow \citet{dpo} to bootstrap the preference dataset ($\hfdata$) from point-wise rewards. Given the reward function, which in this experiment is given by the sentiment classifier, and a dataset of prompts that are prefixes of movie reviews, we proceed as follows. First, for each prompt $\bx \in \dataset$, $M$ responses are sampled from the language model $\piref(\cdot \mid \bx)$\footnote{One can also compare the responses generated by the model with human written responses \citep{summarizehf}.} to form a multiset $Y_{\bx} = \{\by_i\}_{m=1}^M$,\footnote{In this experiment, we generate two responses per prompt.} where $\by_m$ is the $m^{\text{th}}$ response generated for prompt $\bx$. Next, for each
 $i \neq j \in \{1, \dots, M\}$ a tuple $(\bx, \by_w, \by_l)$ is added to $\hfdata$ if $\reward(\by_i) \neq \reward(\by_j)$, where\looseness=-1
\begin{align} \label{eq:bootstrap}
    (\by_w, \by_l) = \begin{cases} (\by_i, \by_j) & \textbf{if } \reward(\by_i) > \reward(\by_j) \\
    (\by_j, \by_i) & \textbf{if } \reward(\by_i) < \reward(\by_j)
    \end{cases}.
\end{align}
While DPO only uses the \emph{order} of responses' rewards, with ODPO we further use the reward model to determine the offset between the two responses and set $\deltar = \log \big(\reward(\by_w) -  \reward(\by_l) \big)$. Following \Cref{eq:margin}, this is equivalent to choosing $\sFunc(\cdot)$ as the $\log$ function and setting $\alpha = 1$.

We vary the KL regularization term $\beta$ in $\{0.1, 0.2, \dots, 1\} \cup \{1, 2, 3, 4, 5\}$\footnote{Within the range of 0 to 1, increased resolution is necessary due to considerable variance in results. However, within the range of 1 to 5, the points are closely clustered, therefore, we increase the step size when covering this range.}
 and sample from the fine-tuned language model $\policy$ two times with two different random seeds.
This gives us $28$ different samples for each method. 
We use these samples to approximate the average probability of the generations having positive sentiment as well as the KL divergence between the fine-tuned model $\policy$ and the initial model $\piref$. 
Ideally, we want not to diverge too much from the SFT policy, while generating movie reviews with a positive sentiment.  

To capture the tradeoff between the KL divergence and achieved reward, we evaluate the two methods based on the Pareto frontier of achieved reward and KL divergence. Concretely, we report the percentage of points on the Pareto set that belong to each method. We compare the performance of ODPO and DPO in $3$ different settings by varying the number of data points in the preference dataset, i.e., $|\hfdata|$. As depicted in \Cref{fig:imdb}, in all experimental setups, ODPO is more effective in generating more samples with positive sentiment, while not diverging too far from $\piref$.\footnote{Further experiments comparing SLiC-HF and ODPO show the same trends \Cref{sec:slic}.} The difference between the two methods is more pronounced with smaller datasets, i.e., when $|\hfdata| = 5000$.\footnote{When using more data,  e.g., $|\hfdata|=10000$, we observe the positive sentiment probability gets closer to the maximum value $1.$, while this comes at the cost of diverging too much from $\piref$; we see an order of magnitude larger $\kl$ with $|\hfdata|=10000$ compared to $|\hfdata|=7500$.}
\begin{figure*}
     \centering
     \begin{subfigure}[b]{0.3\textwidth}
         \includegraphics[width=\linewidth]{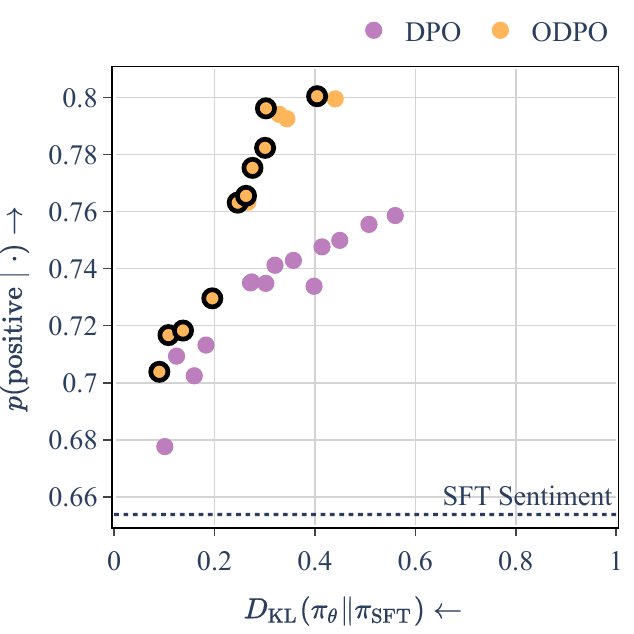}
         \caption{$\deltar = \log \big(\reward(\by_w) - \reward(\by_l)\big)$, $100\%$ of points on Pareto front belong to ODPO.} %
     \end{subfigure}
     ~
     \begin{subfigure}[b]{0.3\textwidth}
         \includegraphics[width=\linewidth]{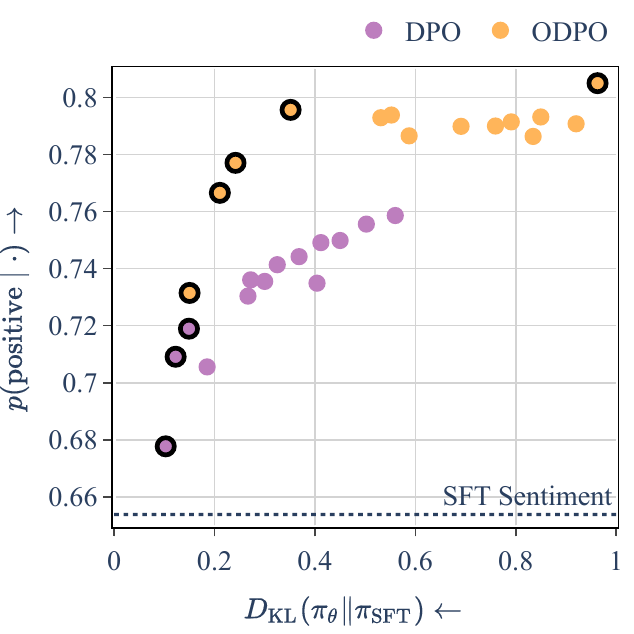}
         \caption{$\deltar = \log \reward(\by_w) - \log \reward(\by_l)$, $63\%$ of points on Pareto front belong to ODPO.} %
     \end{subfigure}
     ~
     \begin{subfigure}[b]{0.3\textwidth}
         \includegraphics[width=\linewidth]{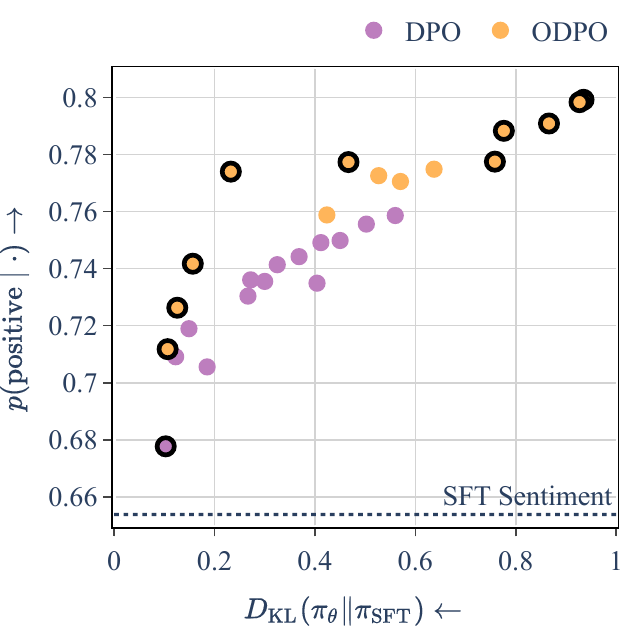}
         \caption{$\deltar = \reward(\by_w) - \reward(\by_l)$, $91\%$ of points on Pareto front belong to ODPO.} %
    \end{subfigure}
    \caption{Investigating the effect of the offset formulation on the performance of ODPO. Scaling the offset with a $\log$ function helps achieve the highest reward values without diverging too much from the SFT model.}
    \label{fig:imdb-cost}
\end{figure*}
\subsection{Toxicity Control}
In this task, our goal is to reduce the toxicity of the generations. We use \texttt{GPT-neo-2.7b} \citep{gpt-neo}\footnote{Specifically, we use \url{https://huggingface.co/EleutherAI/gpt-neo-2.7B}.} as the SFT checkpoint, $\piref$. We adversarially sample the prompts from \textsc{RealToxicityPrompts} \citep{realtoxicityprompts}, where we sample $10000$ prompts that have toxicity scores of more than $0.3$. We generate two preference pairs from $\piref$ for each prompt and compute their rewards using a toxicity classifier.\footnote{We use \url{https://huggingface.co/facebook/roberta-hate-speech-dynabench-r4-target}.} Similar to the previous experiment, we choose the offset as the $\log$ scaled differences between the rewards of preferred and dispreferred responses.

We vary the KL regularization term $\beta \in \{0.05, 0.1, 0.2, 0.3, 0.4, 0.5\}$\footnote{For $\beta > 0.5$, we observe that no toxicity reduction compared to the SFT model; therefore, for this experiment we only test $\beta \leq 0.5$.} and sample from the fine-tuned model $\policy$. We evaluate the samples based on their average toxicity probabilities and the KL divergence between $\policy$ and $\piref$. 

We observe that ODPO significantly outperforms DPO when using a dataset of size $8000$ or $9000$, where all the points on the Pareto front are from ODPO in \Cref{fig:toxicity}. Concretely, in those two setups, DPO fails to meaningfully reduce the toxicity of the generations, as the toxicity scores are all close to the toxicity scores of generations from $\piref$ (shown with a dashed horizontal line). However, ODPO manages to significantly reduce the toxicity of the generations, cutting it almost by half. 
As we increase the size of the dataset, the performance of DPO gets closer to ODPO. 
\begin{figure}
    \centering
    \includegraphics[width=0.95\columnwidth]{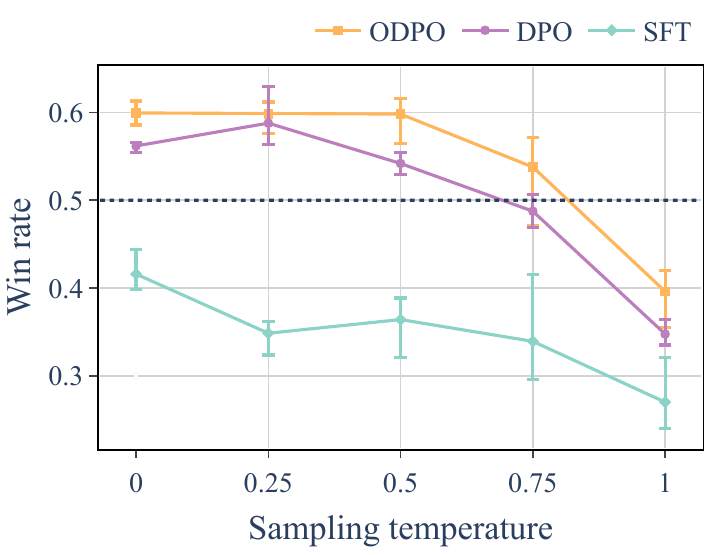}
    \caption{Win rates of generations from models fine-tuned with DPO and ODPO against human-written summaries. On average, ODPO achieves a higher win rate---significantly so in temperatures $0$ and $0.5$.\looseness=-1
    }
    \label{fig:tldr}
\end{figure}
\subsection{Summarization} \label{sec:summarization}
Next, we test our method on a summarization task, where we directly use the scores given by humans to the pair of summaries for designing the offset. We use \textsc{reddit tl;dr} dataset collected by \citet{summarizehf}. We take the subset of this dataset that is assessed by humans for quality using a 7-point Likert scale.\footnote{More specifically, we take the posts under \texttt{relationships} and \texttt{relationship\_advice} subreddits. For more details regarding the datasets used refer to \Cref{sec:dataset}.} 
For each of the two summaries of a given Reddit post, we add a datapoint to $\hfdata$ if the overall quality score of one is greater than the other. 
Moreover, we set the offset between the two summaries as the $\log$ difference between their associated scores by humans. 

The supervised fine-tuned model ($\piref$) is a \texttt{GPTJ-6B} \citep{gpt-j} model that is fine-tuned on human-written forum post summaries.\footnote{Specifically, we use \url{https://huggingface.co/CarperAI/openai_summarize_tldr_sft}} We further fine-tune this model with DPO and ODPO on $\hfdata$.\footnote{Following \citet{dpo}, we set $\beta=0.5$ for both methods.} We use GPT-4 to compare the generations from the fine-tuned models against the reference summary written by humans, as GPT-4 judgments have been shown to strongly correlate with human judgments \citep{dpo}. See \Cref{sec:prompt} for the prompt used for the evaluation.\looseness=-1

We use $100$ test prompts and sample from the fine-tuned models with different temperatures. The win rates of different methods against human written summaries are reported in \cref{fig:tldr}. We observe that both DPO and ODPO improve upon the SFT model. 
At lower temperatures, i.e., $0, 0.25, 0.5$, both methods are on average preferred over human-written summaries. These results are consistent with the results reported in \citep{dpo}. 

Importantly, across all sampling temperatures, the average win rate of ODPO is higher than DPO. Specifically, in temperatures $0$ and $0.5$ ODPO significantly outperforms DPO. We further perform a head-to-head comparison between DPO and ODPO and observe that the win rate of ODPO over DPO is $[51, 50, 62, 48, 57]$ at temperatures $[0, 0.25, 0.5, 0.75, 1.]$ respectively. For qualitative results refer to \Cref{sec:qualitative}.
\section{Ablation Studies}
In this section, we investigate and ablate certain design decisions that were made in the experiments. Specifically, we investigate the effect of the offset hyperparameter $\alpha$, and the scaling function $\sFunc(\cdot)$ on the performance of ODPO. 
\subsection{Ablating the Scaling Function} \label{sec:ablate-f}
In our experiments, we scaled the reward difference between the preferred and dispreferred responses with a $\log$ function. To better understand the effect of this scaling function, we compare it to two other alternatives, (i) using the reward difference without scaling, i.e., setting $\sFunc(\cdot)$ to the identity function, and (ii) using the difference between $\log$ of reward values as the offset. 

We repeat the movie review generation experiment, where the goal is to generate movie reviews with a positive sentiment. For computational efficiency, we choose $|\hfdata| = 5000$. We vary $\beta$ between $\{0.1, 0.2, \dots, 0.9\} \cup \{1, 2, 3, 4, 5\}$ and fine-tune $\policy$ with the aforementioned offsets. 

The results are depicted in \Cref{fig:imdb-cost}. Overall, ODPO outperforms DPO with all three choices of the offset. However, there exist some notable differences between the three offsets. Scaling the offset with the $\log$ function makes the model not diverge too far from $\piref$, while achieving high rewards: While the models trained with $\log$ scaled offset (left plot) achieve high rewards (around $0.8$) with KL of $0.4$, models without $\log$ scaling reach to $0.8$ rewards only when the KL between $\policy$ and $\piref$ is around $1$ (right plot). Scaling each reward value separately lies somewhere in between (middle plot). 
\begin{figure}
    \centering
    \includegraphics[width=0.9\columnwidth]{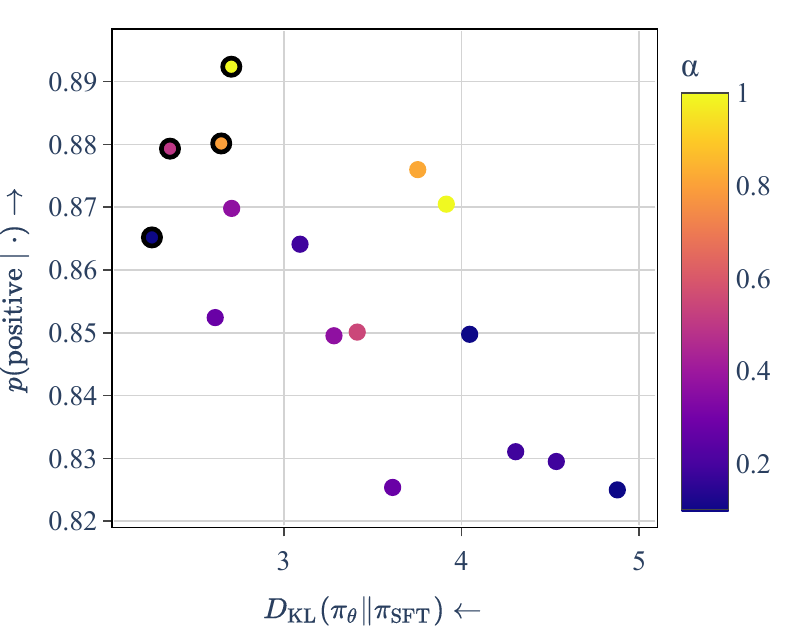}
    \caption{Ablating alpha values for movie review generation. The offset is set to $\deltar = \alpha\, \log \big(r(\by_w) - r(\by_l) \big)$. The highest rewards are associated with models trained with $\alpha=1$, while smaller $\alpha$ leads to lower $\kl$.}
    \label{fig:alpha}
    \vspace{-15pt}
\end{figure}
\subsection{Ablating the Offset Hyperparameter} \label{sec:ablate-alpha}
In the previous experiment, we set $\alpha = 1$. To understand the effect of enforcing an offset in the loss function we experiment with different values of $\alpha$. Repeating the movie review experiment, we vary $\alpha \in \{0.0, 0.1, 0.2, 0.3, 0.5, 0.8, 1.0\}$ and set the offset to $\deltar = \alpha\, \log \big(\reward(\by_w) - \reward(\by_l) \big)$.
Note that when $\alpha=0$ ODPO loss is exactly equal to DPO loss. After fine-tuning $\policy$ on a dataset of size $7500$ with the ODPO loss \Cref{eq:mdpo-loss} (with $\beta = 0.5$), we sample from the models two times with two different random seeds and report the results in \Cref{fig:alpha}. \looseness=-1

We observe that higher values of $\alpha$ can lead to higher reward values at the expense of diverging from the SFT model. Lower values of $\alpha$ on average lead to lower reward values. On the Pareto front of the \Cref{fig:alpha}, points with the highest rewards are associated with models trained with $\alpha=1$, while points with lower $\kl$ are fine-tuned with smaller $\alpha$ values. \looseness=-1
\section{Related Work} \label{sec:related-work}
In this section, we review alternative methods to RLHF for aligning language models to human preferences. 
\paragraph{Sequence-Level Contrastive Methods.}
The intuition that is shared by these methods is simple: they encourage the language model $\policy$ to assign more probability mass to the preferred response compared to the dispreferred response(s). \citet[\textsc{SLiC;}][]{slic-hf} employ a rank calibration loss that requires $\log \policy(\by_w \mid \bx)$ to be greater than $\log \policy(\by_l \mid \bx)$, by an offset $\delta$. Importantly, in their formulation $\delta$ is a hyperparameter and \emph{does not} depend on the the responses $\by_w, \by_l$. Similarly, \citet[\textsc{rrhf;}][]{rrhf} uses a ranking loss without the offset. While intuitive, the objective that contrastive losses are optimizing for, and its connection to the RLHF objective (maximizing the expected reward) is unclear.
\paragraph{DPO and Variants.} DPO proposes an alternative method to optimize the regularized expected rewards without using RL algorithms. Importantly, DPO shares \emph{the objective} with conventional methods for RLHF \citep{summarizehf}, but the optimization is done without training a separate reward model nor using RL algorithms. Since the introduction of DPO, several follow-up studies attempted to improve DPO along different dimensions. To prevent DPO from overfitting to the preference dataset, \citet{ipo} introduce Identity Preference Optimization (IPO). IPO replaces the unbounded function of preference probabilities in DPO loss formulation with the bounded identity function. \citet{kto} propose a method called Kahneman-Tversky Optimisation (KTO) that dispenses the need for paired preference data altogether. KTO's loss function relies on \emph{unpaired} examples that are labeled either as ``good'' or ``bad''. 
\citet{modpo} suggest another variant of DPO specifically for multi-objective alignment.

\paragraph{Our Approach.} ODPO attempts to solve another shortcoming of DPO, which is to treat every preference pair equally and not take into account the extent to which the two responses differ from each other. ODPO's loss requires the estimated reward for the preferred response to be larger than the dispreferred response by an offset that depends on the difference between the quality of the responses. 
\section{Conclusion}
We propose ODPO, a generalization of DPO for aligning language models with human preferences. 
Just as with DPO, ODPO does not rely on a pre-trained reward model and does not require an RL algorithm.
However, in contrast to DPO, ODPO does not treat every preference pair equally and incorporates the extent to which one response should be preferred over another in its loss function. 
Experiments on a variety of tasks suggest that ODPO is more effective than DPO in aligning language models to human preferences.

\section*{Limitations}

\paragraph{Human Preference Data.} 
Not all datasets with human feedback contain judgments regarding the extent to which one response is preferred over another. In our experiments, we focused on tasks and datasets where we had either access to such information or point-wise reward functions. The results presented in the paper provide strong motivation to ask humans to indicate their degree of preference when collecting human feedback data.

\paragraph{Offset values.} Deciding how to scale offset values can depend on the task. In this study, we experimented with offsets based on Likert scores and classifier probabilities. We defer extending ODPO to different tasks for future work.
\section*{Ethical Considerations}
We foresee two main ethical concerns regarding the use of direct preference optimization and, in general, any RLHF method for aligning language models with human preferences. 
First, as with DPO and other RLHF methods, malicious actors can use ODPO to steer the responses to generate harmful or toxic content. However, we must note that we foresee no particular reason for DPO to be more suitable for malicious use cases compared to DPO or other RLHF methods. 
Second, reward functions might inherit unwanted biases from the datasets that they were trained on. Therefore, fine-tuning to increase the expected reward can then inject such biases into the language models’ generations. We view ODPO as a fine-tuning strategy that can work with any given reward function or preference dataset. Therefore, as we develop more accurate and less biased reward functions, or find more inclusive ways of collecting human feedback, less unwanted bias will propagate to the fine-tuning process.\looseness=-1
\section*{Acknowledgements}
We thank Alizée Pace for her feedback on an early version of this paper. Afra Amini is supported by ETH AI Center doctoral fellowship.
\bibliography{custom}

\newpage
\appendix
\onecolumn
\section{Proof of \Cref{theorem:gumbelbt}} \label{sec:gumbel}
\gumbelbt*
\begin{proof}
    We define a random variable $I$ as,
    \begin{equation}
        I \defeq \argmax_{l, w} \{\rtilde_l, \rtilde_w \}
    \end{equation}
    To prove the theorem, it suffices to prove,
    \begin{equation}
        \prob(I = w) = \frac{\exp \big( \rhattheta(\bx, \by_w) \big)}{\exp \big( \rhattheta(\bx, \by_w) \big) + \exp \big( \rhattheta(\bx, \by_l) \big)}
    \end{equation}
    For ease of notation, we define $\gw \defeq \gumbel(\rhattheta(\bx, \by_w), 1)$, $\rhatw \defeq \rhattheta(\bx, \by_w)$, and $\rhatl \defeq \rhattheta(\bx, \by_l)$.
    Then, consider the following manipulation
    \begin{subequations}
    \begin{align}
        \prob (I = w) &= \E_{m \sim \gw} \bigl[ \prob(\rtilde_l < m) \bigr] \\
        &= \int_{-\infty}^{+\infty} \gw(m) \exp \bigl( -\exp(\rhatl-m) \bigr)\, \mathrm{d}m \\
        &= \int_{-\infty}^{+\infty} \exp \bigl(\rhatw -m - \exp(\rhatw - m) \bigr) \exp \bigl( -\exp(\rhatl-m) \bigr)\, \mathrm{d}m \\
        &= \int_{-\infty}^{+\infty} \exp(\rhatw) \exp(-m) \exp \Bigl( -\exp(-m) \bigl(\underbrace{\exp(\rhatw) + \exp(\rhatl}_{\vcentcolon= Z}) \bigr)\Bigr)\, \mathrm{d}m \\
        &= \exp(\rhatw) \int_{-\infty}^{+\infty} \exp(-m) \exp \Bigl( -Z\, \exp(-m)\Bigr)\, \mathrm{d}m \\
        &= \frac{\exp(\rhatw)}{\exp(\rhatw) + \exp(\rhatl)}.
    \end{align}
    \end{subequations}
\end{proof}
\section{Datasets Statistics} \label{sec:dataset}
We used the following datasets for either fine-tuning or prompting language models, which is consistent with the intended use case of the datasets. All the datasets are in English. We refer to the corresponding papers for data collection and postprocessing procedures, as well as the demographics of human annotators.
\paragraph{Sentiment Control.} We train the SFT model on the train set of \textsc{imdb} dataset \citep{imdb}, which consists of $25000$ movie reviews. For fine-tuning with DPO and ODPO, we sample $10000$ prompts from the train set and use the language model to generate continuations. For evaluation, we sample 256 prompts from the test set of \textsc{imdb} dataset.

\paragraph{Toxicity Control.} We sample $10000$ prompts for fine-tuning and $256$ for evaluation from \textsc{RealToxicityPrompts} \citep{realtoxicityprompts}. The dataset is released under Apache-2.0 license.

\paragraph{Summarization.} We take $20000$ posts under \texttt{relationships} and \texttt{relationship\_advice} subreddits in \textsc{reddit tl;dr} dataset \citep{summarizehf} and $100$ posts for evaluation. The dataset is released under a modified MIT license. 

\section{Computational Budget}\label{sec:comp-budget}
For sentiment control experiments, a single fine-tuning and evaluation run takes approximately $20$ minutes on $2$ \texttt{rtx\_4090} GPUs. For toxicity control experiments, a single fine-tuning and evaluation run takes approximately $2$ hours on $2$ \texttt{a100\_40gb} GPUs. For the summarization task, a single fine-tuning and evaluation run takes approximately $15$ hours on $2$ \texttt{a100\_80gb} GPUs. Notably, the reported runtimes are the same for DPO and ODPO, and there is no extra computation cost for ODPO compared to DPO.
\begin{table}[t] 
\centering 
\adjustbox{max width=\textwidth}{%
\begin{tabular}{@{}lcccccc@{}}\toprule
& \multicolumn{2}{c}{$|\hfdata|=5000$} & \multicolumn{2}{c}{$|\hfdata|=7500$} & \multicolumn{2}{c}{$|\hfdata|=10000$} \\
\cmidrule(lr){2-3}
\cmidrule(lr){4-5}
\cmidrule(lr){6-7}
 Method & $\kl(\policy \| \piref) \downarrow$ & Reward $\uparrow$ & $\kl(\policy \| \piref) \downarrow$ & Reward $\uparrow$ & $\kl(\policy \| \piref) \downarrow$ & Reward $\uparrow$ \\ \midrule
 SLiC &  $0.55 [0.52, 0.56]$ & $0.71 [0.71, 0.72]$ & $3.56 [3.24, 3.88]$ & $0.81 [0.81, 0.82]$ & $209.42 [85.94, 333.7]$ & $0.85 [0.84, 0.86]$ \\
 DPO & $0.39 [0.37, 0.4]$ & $0.72 [0.71, 0.73]$ & 4.13 [3.42, 5.27] & 0.83 [0.82, 0.84] & 32.32 [27.57, 36.68] & 0.86 [0.85, 0.86] \\
 ODPO & $\mathbf{0.35} [0.32, 0.39]$ & $\mathbf{0.76} [0.75, 0.78]$ & $\mathbf{3.09} [2.43, 4.04]$ & $\mathbf{0.87} [0.87, 0.88]$ & $\mathbf{18.73} [12.15, 29.65]$ & $\mathbf{0.89} [0.88, 0.89]$ \\
\bottomrule
\end{tabular}}
\caption{Comparing ODPO to SLiC-HF on sentiment control task. In all $3$ experimental setups ODPO achieves lower KL values and higher rewards compared to SLiC.}
\label{tab:slic}
\end{table}
\section{Comparison to SLiC-HF} \label{sec:slic}
We compare DPO and ODPO to SLiC-HF \citep{slic-hf} on the sentiment control task. 
Following \citet{slic-hf}, we set the margin hyperparameter to 1 for SLiC and $\beta$ to $0.5$ for DPO and ODPO. We observe that none of the SLiC runs end up on the Pareto frontier of KL divergence and reward. We report the mean values and $0.9$ confidence intervals for $\kl(\policy \mid\mid \piref)$ and reward in \Cref{tab:slic} for $3$ dataset sizes.

\section{GPT-4 Evaluation Prompt for the Summarization Task} \label{sec:prompt}
Following \citet{dpo} we use the prompt below to evaluate the win rates of generated summaries against human-written summaries. 

\texttt{Which of the following summaries does a better job of summarizing the most important points in the given forum post, without including unimportant or irrelevant details? A good summary is both concise and precise.}

\texttt{Post:}

\texttt{<post>}

\texttt{Summary A:}

\texttt{<summary\_a>}

\texttt{Summary B:}

\texttt{<summary\_b>}

\texttt{FIRST provide a one-sentence comparison of the two summaries, explaining which you prefer and why. SECOND, on a new line, state only "A" or "B" to indicate your choice. Your response should use the format:}

\texttt{Comparison: <one-sentence comparison and explanation>}

\texttt{Preferred: <"A" or "B">
}

\section{Qualitative Results} \label{sec:qualitative}
\begin{table}[ht]
    \centering
    \begin{tabular}{p{3cm}p{11cm}}
    \toprule 
       \textbf{Post}  & We have been dating for two and a half months, but we were friends in high school. We are both divorced, but her kids' dad takes them maybe one day a week. By their bedtime she's too tired to stay up for time for herself. I have more free time because I only have my kids for half the time. 

What can I say to her to help her get over the mentality she has to take care of all her problems alone? I'm willing to invest whatever it takes because we get along so well. Ask me any info that I may have left out.
 \\[0.5cm]
       \textbf{DPO Summary} & dating with benefits is giving my girlfriend anxiety and a mentality of taking care of her problems alone even though she said she wanted a relationship, not a one-man show. Do I need a break-up? Ask me anything.\\ [0.5cm]
       \textbf{ODPO Summary} & Girl is tired when hanging out with me. What can I say to her to help her get over the mentality she has to take care of all her problems alone? We get along so well. Any info that I may have mentioned/left out.\\[0.5cm]
       \textbf{Judgment} & Summary A \textbf{[DPO]} brings up unnecessary details and confusion while Summary B \textbf{[ODPO]} is concise, straightforward and better delivers the post's message. \\ \bottomrule
    \end{tabular}
    \caption{An example of two summaries sampled with temperature 1, and its corresponding judgment from GPT-4. Summaries are shown to the model in random order and are tagged with A or B. For clarity, we add the method in brackets.}
    \label{tab:tldr-temp1}
\end{table}

\begin{table}[ht]
    \centering
    \begin{tabular}{p{3cm}p{11cm}}
    \toprule 
       \textbf{Post}  & I don't really have a question or anything. I just need to vent and maybe reach out to those who have been in similar situations. 

My sister told me yesterday that she's pregnant. At first I was thrilled. But then reality hit me. *I can't have children.*

I'm so happy for my sister and my brother-in-law, I really am. But they weren't planning on having children for at least a few more years. My sister's never really wanted to be a mom. Me, on the other hand, have always lived for it. As a kid when people asked what I want to be when I grow up, I always said, "a mom." 

But, 3 years ago I found out I can never have children. I was devastated but have tried to tough it out, tried to believe it isn't true for years. I'm in a serious relationship and we're planning on getting married in a little over a year and starting to **try** having kids immediately, but we will probably end up adopting. While adopting is wonderful and that thought excites me, I still so badly want to be pregnant one day. 

My sister being pregnant is like a kick in the side. A reminder that I'll never be pregnant. I don't want to be around her, I don't want to be happy for her. I'm jealous. I don't know how to deal with this. I never expected to feel anger towards my sister for getting pregnant. But really, how am I expected to feel about this?

I would never dare tell my sister or anyone else in my family how I feel about this. I wouldn't want to upset anyone and I really don't want to be selfish and make this about me. It's just so hard.
 \\[0.5cm]
       \textbf{DPO Summary} & I have a sister who's pregnant and I feel immensely jealous and sad and I don't know how to deal with my feelings. I don't want to have children.\\ [0.5cm]
       \textbf{ODPO Summary} & My sister is pregnant and I'm jealous and sad even though I'm happy for her and I would love to be pregnant one day. I don't know how to feel. I really just want to vent.\\[0.5cm]
       \textbf{Judgment} & Summary A \textbf{[ODPO]} better captures the overall emotion, situation, and dilemma of the original post, while Summary B \textbf{[DPO]} falsely claims that the poster does not want children, which contradicts the information provided in the original post. \\ \bottomrule
    \end{tabular}
    \caption{An example of two summaries sampled with temperature 0.75, and its corresponding judgment from GPT-4. Summaries are shown to the model in random order and are tagged with A or B. For clarity, we add the method in brackets.}
    \label{tab:tldr-temp0.75}
\end{table}

\begin{table}[ht]
    \centering
    \begin{tabular}{p{3cm}p{11cm}}
    \toprule 
       \textbf{Post}  & We have been dating for two and a half months, but we were friends in high school. We are both divorced, but her kids' dad takes them maybe one day a week. By their bedtime she's too tired to stay up for time for herself. I have more free time because I only have my kids for half the time. 

What can I say to her to help her get over the mentality she has to take care of all her problems alone? I'm willing to invest whatever it takes because we get along so well. Ask me any info that I may have left out.
 \\[0.5cm]
       \textbf{DPO Summary} & My girlfriend has to take care of her kids and her problems because she doesn't have time to do anything for herself. How can I help her?\\ [0.5cm]
       \textbf{ODPO Summary} & My girlfriend of two months is constantly tired from taking care of her kids. How can I help her get over the mentality she has to take care of everything alone?\\[0.5cm]
       \textbf{Judgment} & Summary B \textbf{[ODPO]} is more concise and accurate as it mentions the key point about the poster wanting to help his girlfriend overcome her mentality of taking care of everything by herself, compared to Summary A \textbf{[DPO]} which does not convey this aspect clearly.  \\ \bottomrule
    \end{tabular}
    \caption{An example of two summaries sampled with temperature 0.5, and its corresponding judgment from GPT-4. Summaries are shown to the model in random order and are tagged with A or B. For clarity, we add the method in brackets.}
    \label{tab:tldr-temp0.5}
\end{table}

\begin{table}[ht]
    \centering
    \begin{tabular}{p{3cm}p{11cm}}
    \toprule 
       \textbf{Post}  & I am currently in college and I met this wonderful girl. We have gone on a few dates and everything appeared to be going great. She wanted to taking things quite slow because she was getting over a recent ex boyfriend.

However when I met her roommate she brought up that this girl has a boyfriend who has been acting rather shady towards her. He goes to another college and he has been spending the night "on the couch" at this other girls place frequently. He also has recently been posting on her wall and now I'm rather confused.

I know something is up but I really don't know how to handle this. I would really like to be straight forward with her but she seems dodgy whenever I do. I really am beginning to over think things and I almost feel like I was just there to shove in his face.

I really hate this situation because I feel like I could have a legitimate relationship with this girl.
 \\[0.5cm]
       \textbf{DPO Summary} & Met a girl who has a boyfriend who is spending the night at her place frequently. She has been posting on his wall and I feel like she is hiding something. I don't know how to proceed.\\ [0.5cm]
       \textbf{ODPO Summary} & met girl, she has a boyfriend who is acting shady towards her, she wants to take things slow, I'm confused as to what is going on.\\[0.5cm]
       \textbf{Judgment} & Summary A \textbf{[ODPO]} provides more accurate information about the situation, specifically pointing out that the girl wants to take things slow, and expresses the poster's confusion more effectively, whereas, Summary B \textbf{[DPO]} provides false information about the girl posting on the boyfriend's wall. \\ \bottomrule
    \end{tabular}
    \caption{An example of two summaries sampled with temperature 0.25, and its corresponding judgment from GPT-4. Summaries are shown to the model in random order and are tagged with A or B. For clarity, we add the method in brackets.}
    \label{tab:tldr-temp0.25}
\end{table}

\begin{table}[ht]
    \centering
    \begin{tabular}{p{3cm}p{11cm}}
    \toprule 
       \textbf{Post}  & My best friend and I have been friends for around 5 years now? She and her boyfriend has been together for a few years alr too. She is the kind of person who always put relationship over friendship, and because she's my best friend, I understand her and never said anything. 

I understand when she stop replying me for days when she stays over at her boyfriend's house. I understand when she ask me to leave early because she wanted to spend time with her boyfriend. I understand when she said I'm not allowed to be in contact with her boyfriend.(One time she got angry and ignored me for a month because I called her boyfriend to come over when they had a fight) 

I'm really okay with her being like this. Her boyfriend is an important part of her life and I can accept that. But recently, I can't even get a reply from her. Her text replies are always short and cold. We go to different schools so I don't get to talk/see her everyday. I'm always the one looking for a topic to talk about with her. She only comes looking for me if she and her boyfriend has a fight and she needed somebody to talk to or when her boyfriend is too busy to talk or go out with her. 

I had confronted her about it, but she said that it's not like that, she didn't mean to make me feel this way. But after that it's always back to the same situation and I always end up feeling used and neglected.

Any advice?
 \\[0.5cm]
       \textbf{DPO Summary} & My best friend doesn't reply me anymore and I feel like she doesn't care about me anymore. She always put her boyfriend over me. I'm okay with her being like this, but I'm not okay with feeling like she doesn't care about me anymore. Any advice? \\ [0.5cm]
       \textbf{ODPO Summary} & Best friend of 5 years doesn't reply to my texts anymore and always puts her boyfriend over me. I feel like she doesn't care about me anymore. Any advice?\\[0.5cm]
       \textbf{Judgment} & Summary B \textbf{[ODPO]} is preferred as it encapsulates all the necessary details more concisely and accurately without any unnecessary repetition. \\ \bottomrule
    \end{tabular}
    \caption{An example of two summaries generated with greedy decoding, and its corresponding judgment from GPT-4. Summaries are shown to the model in random order and are tagged with A or B. For clarity, we add the method in brackets.}
    \label{tab:tldr-temp0}
\end{table}

\end{document}